\documentclass[sigconf]{acmart}
\usepackage{subcaption}
\usepackage{booktabs}
\usepackage{mathnotations}
\usepackage{bbm}
\usepackage{tablefootnote}

\AtBeginDocument{%
  }

\makeatletter
\gdef\@copyrightpermission{
  \begin{minipage}{0.3\columnwidth}
    \includegraphics[width=0.90\textwidth]{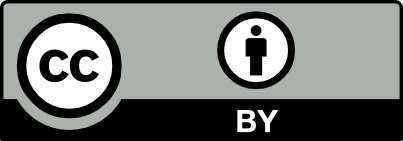}
  \end{minipage}\hfill
  \begin{minipage}{0.7\columnwidth}
    \href{https://creativecommons.org/licenses/by/4.0/}{This work is licensed under a Creative Commons Attribution International 4.0 License.}
  \end{minipage}
  \vspace{5pt}
}
\makeatother

\copyrightyear{2025}
\acmYear{2025}
\setcopyright{cc}
\setcctype{BY}
\acmConference[KDD '25]{Proceedings of the 31st ACM SIGKDD Conference on Knowledge Discovery and Data Mining V.1}{August 3--7, 2025}{Toronto, ON, Canada}
\acmBooktitle{Proceedings of the 31st ACM SIGKDD Conference on Knowledge Discovery and Data Mining V.1 (KDD '25), August 3--7, 2025, Toronto, ON, Canada}
\acmDOI{10.1145/3690624.3709437}
\acmISBN{979-8-4007-1245-6/25/08}

\begin{document}

\title{Session-Level Dynamic Ad Load Optimization using Offline Robust Reinforcement Learning}

\author{Tao Liu$^*$}
\thanks{$^*$The first two authors contributed equally.}
\affiliation{%
  \institution{Meta}
  \streetaddress{}
  \country{Sunnyvale, CA, USA}
}
\email{tliu97@meta.com}

\author{Qi Xu$^*$}
\affiliation{%
  \institution{Meta}
  \streetaddress{}
  \country{Sunnyvale, CA, USA}
}
\email{xuqi@meta.com}

\author{Wei Shi}
\affiliation{%
  \institution{Meta}
  \streetaddress{}
  \country{Sunnyvale, CA, USA}
}
\email{weishi0079@meta.com}

\author{Zhigang Hua}
\affiliation{%
  \institution{Meta}
  \streetaddress{}
  \country{Sunnyvale, CA, USA}
}
\email{zhua@meta.com}

\author{Shuang Yang}
\affiliation{%
  \institution{Meta}
  \streetaddress{}
  \country{Sunnyvale, CA, USA}
}
\email{shuangyang@meta.com}

\renewcommand{\shortauthors}{Tao Liu, Qi Xu, Wei Shi, Zhigang Hua, Shuang Yang}

\begin{abstract}
Session-level dynamic ad load optimization aims to personalize the density and types of delivered advertisements in real time during a user's online session by dynamically balancing user experience quality and ad monetization. Traditional causal learning-based approaches struggle with key technical challenges, especially in handling confounding bias and distribution shifts. In this paper, we develop an offline deep Q-network (DQN)-based framework that effectively mitigates confounding bias in dynamic systems and demonstrates more than 80\% offline gains compared to the best causal learning-based production baseline. Moreover, to improve the framework's robustness against unanticipated distribution shifts, we further enhance our framework with a novel offline robust dueling DQN approach. This approach achieves more stable rewards on multiple OpenAI-Gym datasets as perturbations increase, and provides an additional 5\% offline gains on real-world ad delivery data.

Deployed across multiple production systems, our approach has achieved outsized topline gains. Post-launch online A/B tests have shown double-digit improvements in the engagement-ad score trade-off efficiency, significantly enhancing our platform's capability to serve both consumers and advertisers.
\end{abstract}







\maketitle

\section{Introduction}
\label{sec:intro}
Ad monetization and user engagement are two primary goals of interest of a social networking or e-commerce platform \cite{yan2020ads, carrion2021blending}. By personalizing the quantity and pattern of the advertisements that are incorporated into the user's organic consumption journey, ad load optimization has proven to be an effective approach to achieve the optimal trade-off between these two goals \cite{yan2020ads}.

In general, there are two ways to optimize ad load. The ``static'' approach personalizes the ad load configuration for each user and applies the same ad load throughout the user trajectory, while the ``dynamic'' approach optimizes the ad load configuration in real time, e.g., during an online user session. The latter, which is more challenging, is our focus in this paper.

A session generally refers to a period during which a user is actively engaged with the platform. This can include activities such as browsing newsfeeds, posting updates, commenting on other users' posts, sending messages, etc. The ad load of a session could be tuned by certain product features, e.g., by changing the position of the first ad and the minimum gap between two consecutive ads in a typical newsfeed product \cite{yan2020ads}. Increasing ad load is expected to boost short-term monetization at the cost of hurting user engagement, which eventually could cause damage to long-term monetization opportunities. The key is to control the ad load in real time within an online session to achieve the optimal balance between ad monetization and user engagement. 

\begin{figure}[!htb]
    \centering
    \includegraphics[width=0.5\textwidth]{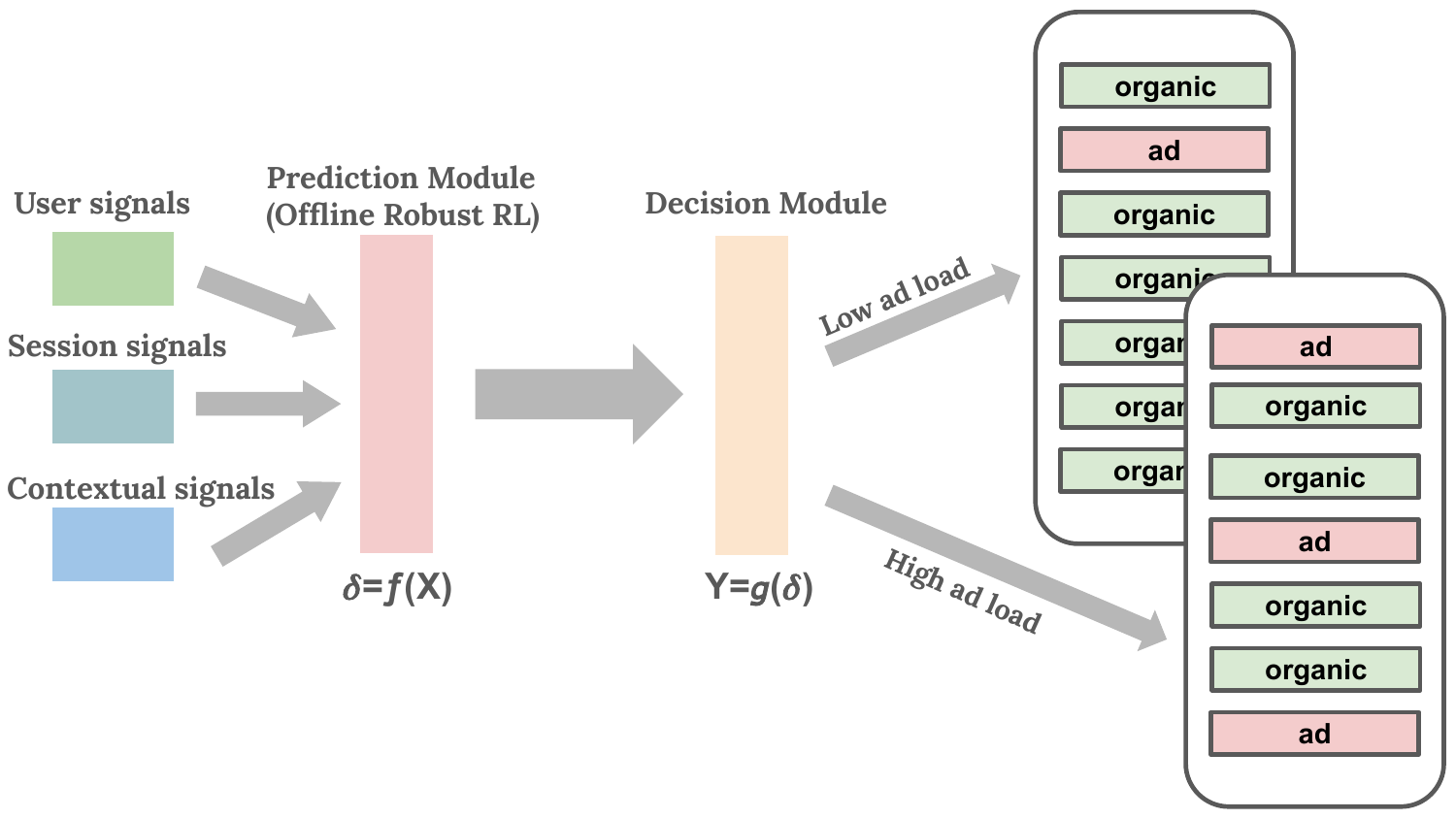}
    \caption{Structure of the session-level dynamic ad load optimization system. A novel offline robust reinforcement learning approach is applied in the prediction module, which generates state-action values as inputs to the decision module (see Fig. \ref{fig:robust-ddqn} for the detailed structure).}
    \label{fig:ad-supply}
\end{figure}

Figure \ref{fig:ad-supply} illustrates a high-level architecture of a session-level dynamic ad load optimization system. It is composed of two major components: 1) prediction and 2) decision. The prediction module consumes a variety of signals $X$ (e.g., user's age, session trigger causes, and short-term user historical interactions) as inputs, and is responsible for providing high-quality predictions for the potential outcome $\delta = f(X)$ with respect to the following possible decisions (e.g., applying low or high ad load). In this case, the outcome could be the ad monetization treatment effect $\Delta^{rev}$ and engagement treatment effect $\Delta^{eng}$ from the to-be-triggered session. On the other hand, the decision module selects the optimal decision based on the predicted outcomes and applies it accordingly to the ongoing session. In this work, according to business needs, the decision module is set to choose top individuals based on $\Delta^{rev} + \alpha \times\Delta^{eng}$, where $\alpha > 0$ is an adjustable hyperparameter. The prediction module plays a key role in the system since the decision module is relatively straightforward and depends on the former.

The prediction problem is highly counterfactual since we can only observe the direct outcome of one treatment at a time (i.e., low ad load or high ad load), but not the treatment effect (i.e., the difference in engagement and monetization between a ``treatment'' group and a ``control'' baseline). Naturally, the prediction module is usually powered by causal learning models such as meta-learners \cite{liu2023explicit, wu2023dnet,kunzel2019metalearners, nie2021quasi}.
When it comes to session-level dynamic ad load optimization, the real-time dynamics pose severe challenges that traditional causal learning-based approaches may struggle with. 
\begin{itemize}
    \item \textbf{Confounding bias.} The confounding bias results from failing to control for the cause variables that one should have controlled for \cite{elwert2014endogenous}. In our case, treatments in previous sessions could have indirect impacts on features, treatment, and outcomes in the current session, which acts as part of hidden confounders. These sequential impacts exacerbate confounding bias in our session-level dynamic ad load optimization system. 

    \item  \textbf {Distribution shifts.} User behavior will change over time, which may bring severe Y|X-shifts (conditional relationships between the outcome and covariate) and significant performance drops due to missing variables and hidden confounders \cite{liu2023need}.
\end{itemize}
In the literature, debiased causal modeling has been developed to alleviate the confounding bias via Neyman-orthogonal moments and cross-fitting \cite{chernozhukov2018double, tang2022debiased}, but unfortunately does not perform well in our dynamic system due to sequential impacts of previous treatments. Therefore, there remain open research questions on how to better mitigate confounding bias in our dynamic system and alleviate the challenge of distribution shifts.

In this paper, we seek to systematically address these challenges. Our contributions include the following:

\begin{itemize}
\item We first formulate the session-level dynamic ad load optimization problem as a Markov decision process (MDP) and prototype an offline deep Q-network (DQN) approach. The challenge of confounding bias is mitigated by introducing previous actions (i.e., ad load decisions in previous sessions) as part of current states, leveraging the flexibility of RL modeling. The preliminary experiments show promising offline gains of our approach compared to current meta-learner methods, with 80\%+ improved AUCC (area under cost curve \cite{du2019improve}). To the best of our knowledge, our work is the first RL solution for session-level dynamic ad load optimization. 

\item We further develop a novel offline robust dueling DQN approach aiming to alleviate potential distribution shift issues in session-level dynamic ad load optimization since robust RL optimizes a well-performing policy that is robust against model mismatch. We first demonstrate the proposed approach has more robust behaviors with slow cumulative reward decay as perturbations increase in CartPole-v1 and LunarLander-v2 public data (similar properties as production data). Additionally, on session-level production data, our preliminary experiments show that the proposed robust dueling DQN improves test AUCC by ~5\% and ~25\% compared with non-robust dueling DQN and meta-learner.

\item Our framework has been deployed to multiple production systems and achieved outsized topline business gains. From the post-launch A/B test on live traffic, we observe, on average, a double-digit improvement in engagement-ad score trade-off efficiency (e.g., a double-digit reduction in engagement loss when increasing ad score to the same level or vice versa). It shows that the proposed framework has significantly improved our platform's capability to serve both consumers and advertisers effectively.
\end{itemize}

\section{Preliminaries and Related Works}
A Markov decision process (MDP) is represented by a tuple \\
$(\Sc, \Ac, P, r, \gamma, \rho)$, where $\Sc$ is the state space, $\Ac$ is the action space, $P: \Sc \times \Ac \to \Delta(\Sc)$ is the transition kernel, $r: \Sc \times \Ac \to [0, 1]$ is the reward function, $\gamma \in [0, 1)$ is the discount factor, and $\rho \in \Delta(\Sc)$ is the initial state distribution. Note that $\Delta(\mathcal{X})$ represents a $(|\mathcal{X}|-1)$-dimensional probability simplex, where $\mathcal{X}$ could be $\mathcal{S}$ or $\mathcal{A}$.
Given any stationary policy $\pi: \Sc \to \Delta(\Ac)$, its value function is $V_{P}^{\pi}(s) := \Eb_{P, \pi}\left[\sum_{t=0}^{\infty} \gamma^t r(s_t, a_t) | s_0 = s\right], \forall s \in \Sc$. With a slight abuse of notation, we denote $V_{P}^{\pi}(\rho) := \Eb_{s \sim \rho}[V_{P}^{\pi}(s)]$. We can similarly define the state-action value function as $Q_{P}^{\pi}(s, a) := \Eb_{P, \pi}\left[\sum_{t=0}^{\infty} \gamma^t r(s_t, a_t) | s_0 = s, a_0 = a\right]$, the advantage function as \\ $A_{P}^{\pi}(s, a) := Q_{P}^{\pi}(s, a) - V_{P}^{\pi}(s)$, and the state-action visitation distribution as $d_{\rho}^{\pi, P}(s, a) := (1 - \gamma) \Eb_{P, \pi} [\sum_{t=0}^{\infty} \gamma^t \mathbbm{1}(s_t = s, a_t=a)|s_0 \sim \rho]$. When the context is clear, we represent the state visitation distribution as $d_{\rho}^{\pi, P}(s) := (1 - \gamma) \Eb_{P, \pi} \left[\sum_{t=0}^{\infty} \gamma^t \mathbbm{1}(s_t = s)|s_0 \sim \rho\right]$, which is the marginal distribution of the state-action visitation, i.e., $d_{\rho}^{\pi, P}(s) = \sum_{a \in \Ac} d_{\rho}^{\pi, P}(s, a)$.

In the remainder of this section, we demonstrate more preliminaries and related works about offline RL (Section \ref{subsec:offline-rl}), robust RL (Section \ref{subsec:robust-rl}), area under cost curve metric (Section \ref{subsec:auuc}), and ad allocation (Section \ref{subsec:ad-allo}).

\subsection{Offline Reinforcement Learning}
\label{subsec:offline-rl}
Instead of interacting with the environment and generating additional transitions using the behavior policy $\pi_{\beta}$, offline RL learns the optimal policy of an MDP with a \textit{static} dataset $\Dc_{P} = \left\{\left(s_i, a_i, r_i, s_i'\right)\right\}_{i=1}^N$, where $(s_i, a_i) \sim d_{\rho}^{\pi_{\beta}, P}(\cdot, \cdot)$ and $s_i' \sim P(\cdot|s_i, a_i)$. In principle, most off-policy RL algorithms could be used as offline RL approaches by using an offline dataset $\Dc_{P}$ to prefill the data buffer without additional online exploration \cite{levine2020offline}. 

The major challenge of offline RL is the distribution shift between the state-action visitation distribution of the behavior policy and that of the learned policy, which could lead to severe overestimation for the out-of-distribution (OOD) state-action pairs \cite{fujimoto2019off, kumar2019stabilizing}. To alleviate this issue, a series of model-free offline RL algorithms have been proposed, including policy constraint approaches (i.e., constraining the learned policy to lie close to the behavior policy or its support) \cite{fujimoto2019off, kumar2019stabilizing} and uncertainty-based approaches (i.e., penalizing the Q-values of OOD actions) \cite{kumar2020conservative, kostrikov2021offline}. Additionally, another stream of works is based on model-based offline RL algorithms \cite{kidambi2020morel, yu2020mopo}, which modify the MDP transition learned from data to induce conservative behavior.

\subsection{Robust Reinforcement Learning}
\label{subsec:robust-rl}
The robust Markov decision process (RMDP) formulation differs from the standard MDP in that it takes into account a set of transitions instead of a single transition. We denote RMDP as a tuple $(\Sc, \Ac, \Pc, r, \gamma, \rho)$, where $\Pc$ is a set of transitions known as the \textit{uncertainty set} that is typically defined as
\begin{align}
\label{eqn:uncer-set}
    \Pc = \bigotimes_{s, a} \Pc_{s, a}, \ \Pc_{s, a} = \left\{P_{s, a} \in \Delta(\Sc): d(P_{s, a}, P_{s, a}^0) \leq \delta \right\}.
\end{align}
The uncertainty set $\Pc$ (Equation (\ref{eqn:uncer-set})) follows a key $(s, a)$-rectangularity condition that is commonly assumed since the introduction of RMDPs \cite{iyengar2005robust, nilim2005robust}. Additionally, $P_{s, a}^0$ is the nominal stationary transition of the training environment, $d(\cdot, \cdot)$ is some divergence metric between probability distributions, and $\delta > 0$ is the radius to control the level of perturbations around the nominal transition. 

The robust value function is defined as \cite{iyengar2005robust, nilim2005robust}
\begin{align}
    V_{\Pc}^{\pi}(s) := \inf_{P \in \Pc} V_{P}^{\pi}(s).
\end{align}
The goal of RMDP is to learn an optimal robust policy that achieves the optimal worst-case performance over all possible transitions in the uncertainty set, i.e., 
\begin{align}
    \sup_{\pi} V_{\Pc}^{\pi}(\rho) = \sup_{\pi} \inf_{P \in \Pc} \Eb_{s \sim \rho} [V_{P}^{\pi}(s)].
\end{align}

The corresponding robust Bellman operator $\Tc_{\Pc}: \Rb^{\Sc \times \Ac} \to \Rb^{\Sc \times \Ac}$ is
\begin{align}
    (\Tc_{\Pc} Q)(s, a) = r(s, a) + \gamma \inf_{P \in \Pc_{s, a}} \Eb_{s' \sim P_{s, a}} \max_b Q(s', b).
\end{align}
Since $\Tc_{\Pc}$ is a contraction mapping in the infinity norm \cite{iyengar2005robust}, the Bellman optimality equation for RMDPs is $Q^* = \Tc_{\Pc} Q^*$, where $Q^*$ is its unique optimal solution according to the Banach fixed-point theorem.

\subsection{Area Under Cost Curve Metric}
\label{subsec:auuc}
To evaluate return on investment (ROI), current production adopts the area under cost curve (AUCC) metric \cite{du2019improve}, which is a variant of the area under uplift curve (AUUC) metric \cite{rzepakowski2010decision}.
AUUC is a common metric to measure heterogeneous treatment effect in uplift modeling, which is obtained by ranking individuals in order to choose the top most responsive individuals \cite{gutierrez2017causal}. Traditional AUUC adopts population as the x-axis and one-dimensional treatment effects as the y-axis. However, there are two-dimensional treatment effects in session-level dynamic ad load optimization, i.e., ad monetization and user engagement, which generally show correlated trade-offs, e.g., increasing ad load could potentially bring more ad monetization meanwhile hurting user engagement due to the reduction of organic content. Therefore, we utilize AUCC to evaluate two-dimensional treatment effects with trade-offs to reflect ROI, where normalized engagement loss is used as the x-axis and normalized monetization gain is adopted as the y-axis. The AUCC value represents how well we distinguish the sessions based on their sensitivity to ad load change, which could be utilized for different business needs, e.g., increasing ad load to the sessions that are the most sensitive to monetization increase while the least sensitive to user engagement change. 

Although RL algorithms focus on long-term objectives, it is still reasonable to measure them under the AUCC metric. The main idea is to treat RL models as a prediction module to estimate the long-term reward on monetization and engagement, and decouple it from the decision module which focuses on making ad load decisions to different sessions based on the output from prediction layers. On the other hand, when choosing the discount factor as 0, RL is reduced to a neural contextual bandit, which also shares a similar architecture as the treatment-agnostic representation network in causal learning literature \cite{shalit2017estimating}.

\subsection{Ad Allocation}
\label{subsec:ad-allo}
There is a series of literature on related ad allocation problems in social networking and e-commerce applications in industry \cite{yan2020ads, carrion2021blending, liao2022cross, liao2022deep, wang2022hybrid, sagtani2024ad}, focusing mainly on how to place a fixed number of ads and organic content locations. Several approaches have been proposed to find optimal positions via constrained optimization problem \cite{yan2020ads, chen2022hierarchically}, multi-objective optimization \cite{carrion2021blending}, and end-to-end RL \cite{liao2022cross, liao2022deep, wang2022hybrid, rafieian2023optimizing}. 

However, our ad load optimization problem distinguishes from ad allocation from three perspectives. Firstly, ad load optimization is built on top of the existing mechanism to make further personalization and optimization, e.g., the highest position of the first ad or the minimum distance between two consecutive ads. As a comparison, Yan et al. \cite{yan2020ads} adopt a fixed highest position and minimum distance as a fixed rule and some others do not consider ad load directly. 
Secondly, many existing works on ad allocation seem to have overlooked the strong carry-over effect, and are therefore incapable of capturing how historical treatment affects future observations, including state representations and treatment outcomes.
Thirdly, these works also cannot tackle the dynamics of user behavior, making them less robust to distribution drift without frequent retraining as user behavioral patterns change over time.

\section{Problem Formulation and Methods}
In this section, we start by introducing how to formulate the session-level dynamic ad load optimization problem into a robust MDP (Section \ref{subsec:prob-form}). We then demonstrate the details of the proposed offline robust dueling DQN approach and its related theoretical guarantees under function approximation in Section \ref{subsec:method}.

\subsection{Problem Formulation}
\label{subsec:prob-form}

The session-level dynamic ad load optimization problem is formulated as a robust MDP $(\Sc, \Ac, \Pc, r, \gamma)$, where the elements are defined as follows.
\begin{itemize}
    \item \textbf{State space $\Sc$.} A state $s \in \Sc$ consists of user-level features (e.g., age, tenure, country, etc.), session-level features (e.g., user behavior within the last X hours, time bucket of the day, time since the last session, etc.), and actions in previous sessions. For preprocessing, we apply one-hot encoding for categorical states and standardization for continuous states. 
    \item \textbf{Action space $\Ac$.} An action $a \in \Ac$ is the decision of ad position on the current session. In our scenario, we have a discrete action space that includes two types of actions: low ad load assignment and high ad load assignment.
    \item \textbf{Uncertainty set $\Pc$.} 
    User behaviors will shift over time, which means transitions will also change temporally. Therefore, we consider a set of transitions instead of a single transition. We will specify the choice of uncertainty set $\Pc$ in Section \ref{subsec:method}. Note that we have $\delta=0$ for non-robust MDPs and the corresponding uncertainty set is reduced to a singleton, i.e., $\Pc_{s, a} = \{P^0(s, a)\}$.
    \item \textbf{Reward $r$.} After the agent takes an action in one state, some ad monetization $r^{rev}$ and user engagement $r^{eng}$ signals can be received as feedback. In this paper, we calculate the reward function based on the linear scalarization between normalized $r^{rev}$ and normalized $r^{eng}$, i.e., 
    \begin{align}
        r(s, a) = r^{rev}(s, a) + \alpha r^{eng}(s, a), 
    \end{align}
    where $\alpha > 0$ is a predefined weight. We leave the discussion of several nonlinear scalarizations and linear scalarizations with unknown weights in Section \ref{sec:concl}.
    \item \textbf{Discounted factor $\gamma$.} The discount factor $\gamma \in [0, 1)$ strikes a balance between the short-term and long-term rewards.
\end{itemize}

Note that the above RL formulation targets the prediction module where we solve $\sup_{\pi} V_{\Pc}^{\pi}(\rho)$ based on the offline dataset $\Dc_{P^0} = \left\{\left(s_i, a_i, r_i, s_i'\right)\right\}_{i=1}^N$ as the solution. The decision module is designed by ranking users $\Delta^{rev} + \alpha \Delta^{eng}$ to select suitable sessions to adjust the ad-load strategy, which satisfies our business goal of maximizing the monetization gain and minimizing the engagement loss.

\subsection{Methodology}
\label{subsec:method}
Since we first want to verify whether RL can become a feasible solution to the session-level dynamic ad load optimization problem, we conduct experiments on offline datasets with the same training and test distributions. In this scenario without distribution shifts, the classic offline deep Q-network (DQN) \cite{mnih2015human} is sufficient. Compared with online DQN, offline DQN utilize $\Dc_{P^0} = \left\{\left(s_i, a_i, r_i, s_i'\right)\right\}_{i=1}^N$ to pre-populate the replay buffer and sample a batch $B \subset \Dc_{P^0}$ to update parameters of Q-networks via the loss function
\begin{align}
    L_1(B, \theta) = \frac{1}{|B|} \sum_{i=1}^{|B|} \left(r_i + \gamma \max_{a_i'} Q_{\theta^{target}}(s_i', a_i') - Q_{\theta}(s_i, a_i)\right)^2,
\end{align}
where $\theta^{target}$ only updates every $C$ step ($C \in \mathbb{N}^+$) and are held fixed between individual updates.

Before conducting online A/B testing, the current pipeline will evaluate the model performance on different rounds (weeks) of data with different distributions, which can be regarded as a proxy of dynamic online environments. However, due to the performance degradation caused by distribution shifts and the lack of robustness of existing methods, current algorithms may need to be retrained in a relatively short period, which is expensive and laborious. Therefore, we are eager to employ algorithms with more robust performance against distribution shifts in production. 

If we have a high-fidelity simulator to evaluate as in \cite{fu2020d4rl}, then classical offline RL algorithms (e.g., batch constrained Q-learning (BCQ) \cite{fujimoto2019off} and conservative Q-learning (CQL) \cite{kumar2020conservative}) will be good candidates to alleviate distribution shifts. However, since online A/B testing is expensive, the newly proposed approach has to be first evaluated by the AUCC metric over different weeks of offline data to reflect ROI, where OOD data are inevitable. AUCC metric requires the information of trained Q-values instead of learned policy, which is different from classic offline RL with a simulator for evaluation \cite{fu2020d4rl} and makes BCQ and CQL perform badly (details in Section \ref{subsubsec:-data}). Therefore, we propose a new offline robust dueling DQN approach to alleviate distribution shifts and fit into the AUCC metric.

Since session-level dynamic ad load optimization has continuous features, a computationally feasible RMDP algorithm is required in a large state space. Unfortunately, most tabular RMDP algorithms based on R-contamination uncertainty set \cite{wang2021online} and $l_p$-norm uncertainty set \cite{kumar2022efficient} are computationally infeasible. This is because they require calculating the minimum of the value functions over the entire state space (R-contamination) or mean, median, average peak of value functions depending on the choice of $p$ ($l_p$-norm). While Wasserstein distance uncertainty set \cite{kuang2022learning} and $f$-divergence uncertainty set \cite{panaganti2022robust} are feasible for large-scale RMDP, they either suffer from a lack of theoretical guarantees or long training time with instabilities by calculating optimal dual variable for each state-action pair. Zhou et al. \cite{zhou2023natural} propose two new uncertainty sets, i.e., double sampling (DS) uncertainty set and integral probability metric (IPM) uncertainty set, to make large-scale online RMDP computationally tractable with theoretical convergence guarantees and superior empirical performance. Although DS uncertainty sets are not applicable in offline RL due to their dependence on the existence of a simulator in training, IPM uncertainty sets have the potential to be applied to offline RL.

Specifically, given some function class $\Fc$, IPM is defined as $d_{\Fc}(p, q) := \sup_{f \in \Fc} \{p^T f - q^T f\} \geq 0$ \cite{muller1997integral} for any two probability distributions $p$ and $q$. Then its corresponding uncertainty set is $\Pc$ with
\begin{align}
    \Pc_{s, a} = \{q: d_{\Fc}(q, P_{s, a}^0) \leq \delta, \ \sum_{s \in \Sc} q(s)=1\},
\end{align}
where $P_{s, a}^0(s_{t+1} | s_t, a_t)$ is defined as the state transition probability within the offline training dataset and $t$ is the index for the session. Following \cite{zhou2023natural, kumar2022efficient}, we relax the domain $q \in \Delta(\Sc)$ to $\sum_{s \in \Sc} q(s) = 1$, which provokes no relaxation error for small $\delta$ if $\min_{s'} P_{s, a}^0(s') > 0$. 

In the remaining section, we first derive the closed-form formula of the empirical robust Bellman operator with the IPM uncertainty set under linear function approximation and then generalize it to general function approximation.

\subsubsection{Linear Function Approximation.}
Due to the property of large state space, function approximation of some value functions is required. We start by considering linear function approximation for $\max Q =: V$ and $Q$ separately, i.e., $V_w = \Phi w$ and $Q_{\theta} = \Psi \theta$, where $\Phi \in \Rb^{|\Sc| \times d}$ and $\Psi \in \Rb^{|\Sc||\Ac| \times d}$ are feature matrix with rows $\phi^T(s)$ and $\psi^T(s)$. The corresponding linear function class for $\max Q$ (i.e., $V$) is
\begin{align}
\label{eqn:func-class}
    \Fc := \left\{s \mapsto \phi(s)^T w: w \in \Rb^d, \|w\| \leq 1, \phi(s)^T w \in [0, \frac{1}{1-\gamma}]\right\}.
\end{align}
Without loss of generality, we assume $\Phi$ has full column rank given $d << |\Sc|$ and the first coordinate of $\phi(s)$ be 1 for any $s$, representing the linear regressor's bias term. Then we will have Proposition \ref{prop:ipm} similar to \cite{zhou2023natural} but different definition of $V_w$.
\begin{proposition}
\label{prop:ipm}
    For the IPM uncertainty set with $\Fc$ in (\ref{eqn:func-class}), we have $\inf_{P \in \Pc_{s, a}} P^T V_w = (P_{s, a}^0)^T V_w - \delta \|w_{2:d}\|$.
\end{proposition}
Proposition \ref{prop:ipm} illustrates that the robustness can be transferred to regularization without the bias parameter in the IPM uncertainty sets, which are computationally tractable for large state space and have almost no additional computational burden compared to the non-robust approach. Additionally, the empirical robust Bellman operator 
\begin{align}
\label{eqn:emp-bell}
(\hat{\Tc}_{\Pc} V_w)(s, a, s') := r(s, a) + \gamma V_{w}(s') - \gamma \delta \|w_{2:d}\|
\end{align}
is an unbiased estimate for the robust Bellman operator, and the contraction behavior of $\Tc_{\Pc}$ for IPM uncertainty set can be established in Proposition \ref{prop:contract} similar to \cite{zhou2023natural}.
\begin{proposition}
\label{prop:contract}
    For IPM uncertainty set with radius $\delta < \\
    \lambda_{min}\left(\Phi^T \diag(\nu) \Phi\right) \frac{1-\gamma}{\gamma}$, there exists $\beta < 1$ that $\Tc_{\Pc}$ is a $\beta$-contraction mapping w.r.t. norm $\|\cdot\|_{\nu}$, where $\nu$ is any state-action distribution and $\lambda_{min}$ represents the minimum eigenvalue.
\end{proposition}

In order to meet the requirement that $\max Q$ and $Q$ both belong to linear function classes, we adopt a dueling network architecture \cite{wang2016dueling} without a nonlinear activation function as shown in Figure \ref{fig:robust-ddqn}. We set
\begin{align}
    Q_{\theta}(s, a) = V_w(s) + A_u(s, a) - \max_b A_u(s, b), 
\end{align}
{where $A_u := \Gamma u, \Gamma \in \Rb^{|\Sc||\Ac| \times d}$.}
Then for $a^* = \argmax_a Q_{\theta}(s, a) = \argmax_b A_u(s, b)$, we have $Q_{\theta}(s, a^*) = V_w(s)$.
\begin{figure}[!htb]
    \centering
    \includegraphics[width=0.49\textwidth]{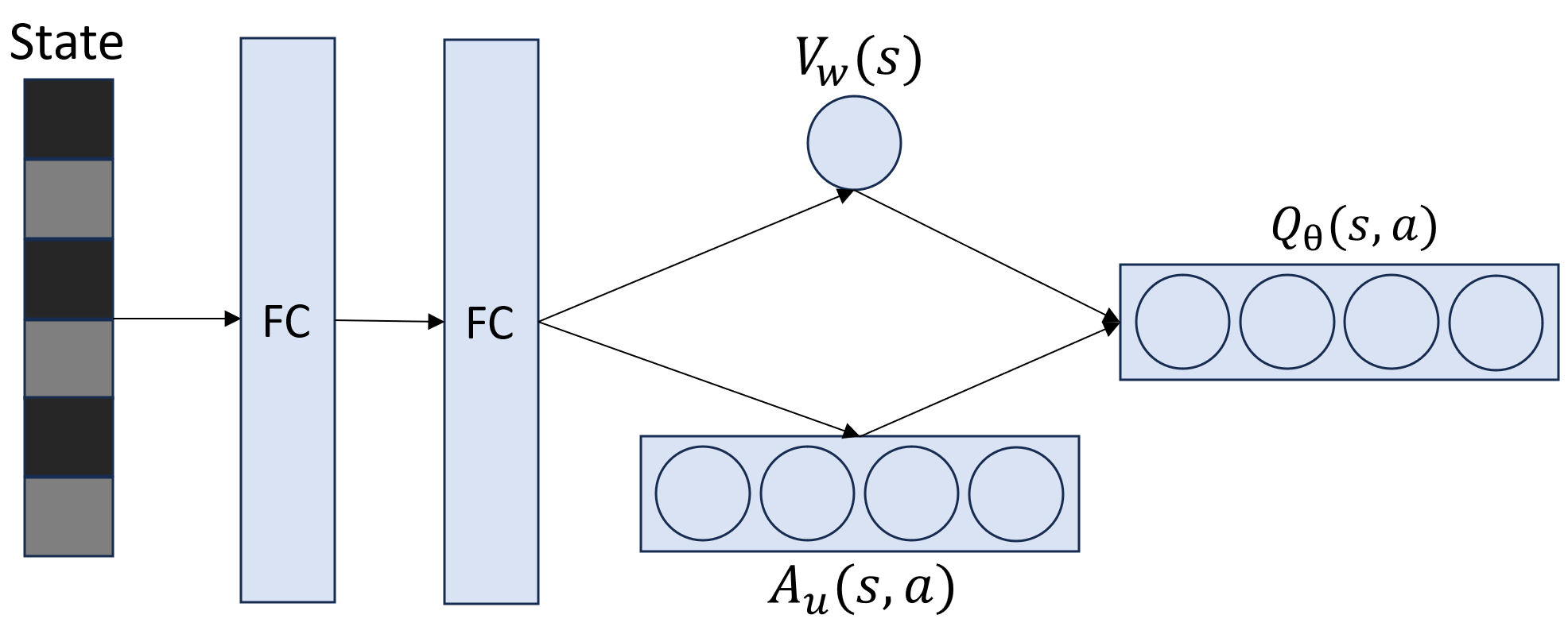}
    \caption{Structure of robust dueling DQN. Robustness is incorporated through the empirical robust Bellman operator (Equations (\ref{eqn:emp-bell}) and (\ref{eqn:general-emp-bell})).}
    \label{fig:robust-ddqn}
\end{figure}

\subsubsection{General Function Approximation.} For general function approximation, we adopt a dueling DQN as a backbone and optimize the loss function following the closed-form of empirical robust Bellman operator (Equation (\ref{eqn:emp-bell}))
\begin{align}
\label{eqn:general-emp-bell}
    L_2(B, \theta) = \frac{1}{|B|} \sum_{i=1}^{|B|} \left(r_i + \gamma V_{w^{target}}(s_i') - \gamma \delta \|w^{target}_{2:d}\| - Q_{\theta}(s_i, a_i)\right)^2.
\end{align}
There are two choices for the regularization term from different perspectives. i) Except for the last linear layer, all previous layers are adopted as a dynamic feature matrix, so we only regularize weights of the last linear layer. ii) Based on the Taylor series approximation, we utilize the initial gradient as a fixed feature matrix, so we regularize all parameters except for the bias parameter in the last layer. In general, the first choice performs better for online RL since the feature matrix has learned enough information due to exploration, while offline RL prefers the second choice since the learned feature matrix could be biased.

\section{Experimental Results}
In this section, we first illustrate data analysis and corresponding metrics of session-level dynamic ad load production data and OpenAI Gym public data (Section \ref{subsec:data-ana}). We then verify offline DQN can become a feasible and effective approach for session-level dynamic ad load optimization problems via production data from the same distribution (Section \ref{subsec:prod-same}). Next, we demonstrate the robust performance of the proposed robust dueling DQN against distribution shifts on both public data and session-level production data (Section \ref{subsec:data-diff}). Finally, we illustrate significant improvements in online deployment on the company's newsfeed platform (Section \ref{subsec:online}).

\subsection{Data Analysis and Metrics}
\label{subsec:data-ana}
We collect the session-level dynamic ad load production dataset by running a combination of a uniformly random policy and a rule-based policy on the company's newsfeed platform. To avoid feature bias introduced by the feedback loop of the ad load treatment, we adopt a Cookie-Cookie-Day (CCD) framework \cite{hohnhold2015focusing} to shuffle users every day in the experiment arm, where rectangles of different colors represent different users. A uniformly random policy is applied during the CCD period for treatment groups, while a rule-based policy is adopted during the non-CCD period.

To numerically analyze the confounding bias (e.g., X|T-shifts) and time-wise user behavior shifts (e.g., Y|X-shifts) in session-level data, we investigate the production data for both challenges, as shown in Figure \ref{fig:shift}. 
For confounding bias, we could clearly observe that the real-time signals (X) are biased under different previous treatments. With a high ad load in the previous session, users' recent consumption of organic posts decreases, and the time gap from the last interaction increases compared to a previous low ad load, both of which indicate a depression in user engagement due to the high ad load. Additionally, we note that treating previous actions as part of the current state has a positive impact on performance improvement in experiments, which implicitly illustrates the impact of confounding bias.
On the other hand, the user behavior against time also showed a clear trend of shifting within all the cohorts (Y|X shift). We select the session trigger cause as X which represents the fetch causes of the current triggered session, which is grouped for calculating the mean of user engagement (time spent within the session) as Y between morning and afternoon. 
{Specifically, \textit{auto}, \textit{back\_button}, \textit{cold\_start}, \textit{manual}, and \textit{warm\_start} represent automatically refreshing by the app, clicking the back button on Android, being inactive at the backend, manually refreshing the app, and being in an active state at the backend, respectively.}
Considering fetch causes are independent of time, it could be observed that there is a clear shift from morning to afternoon in user engagement (Y) for all the fetch cause cohorts. 

\begin{figure}
\begin{subfigure}{0.39\textwidth}
    \centering
    \includegraphics[width=\textwidth]{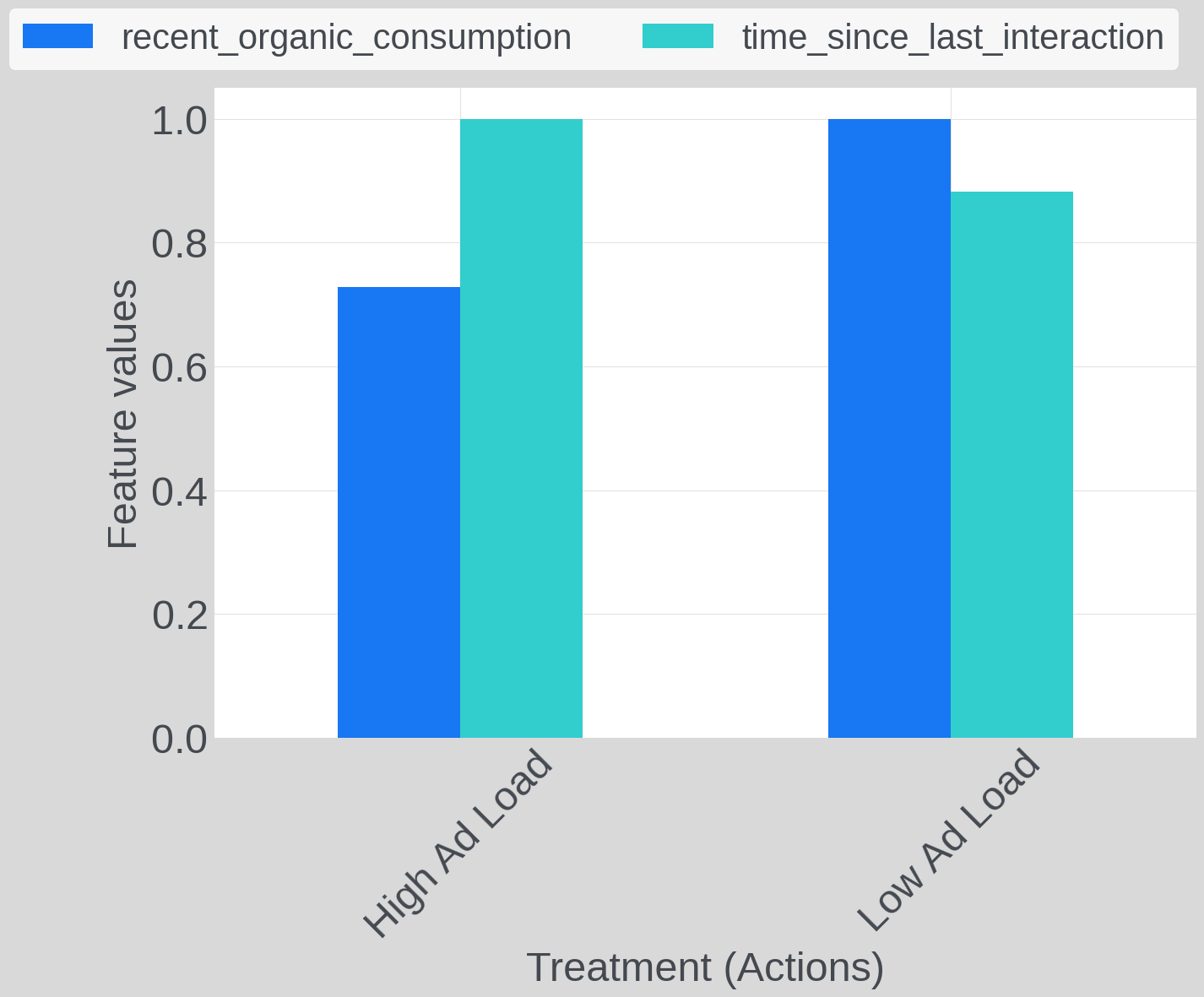}
    \caption{Mean values of different features (X) by different treatments (T)}
\end{subfigure}
\begin{subfigure}{0.39\textwidth}
    \centering
    \includegraphics[width=\textwidth]{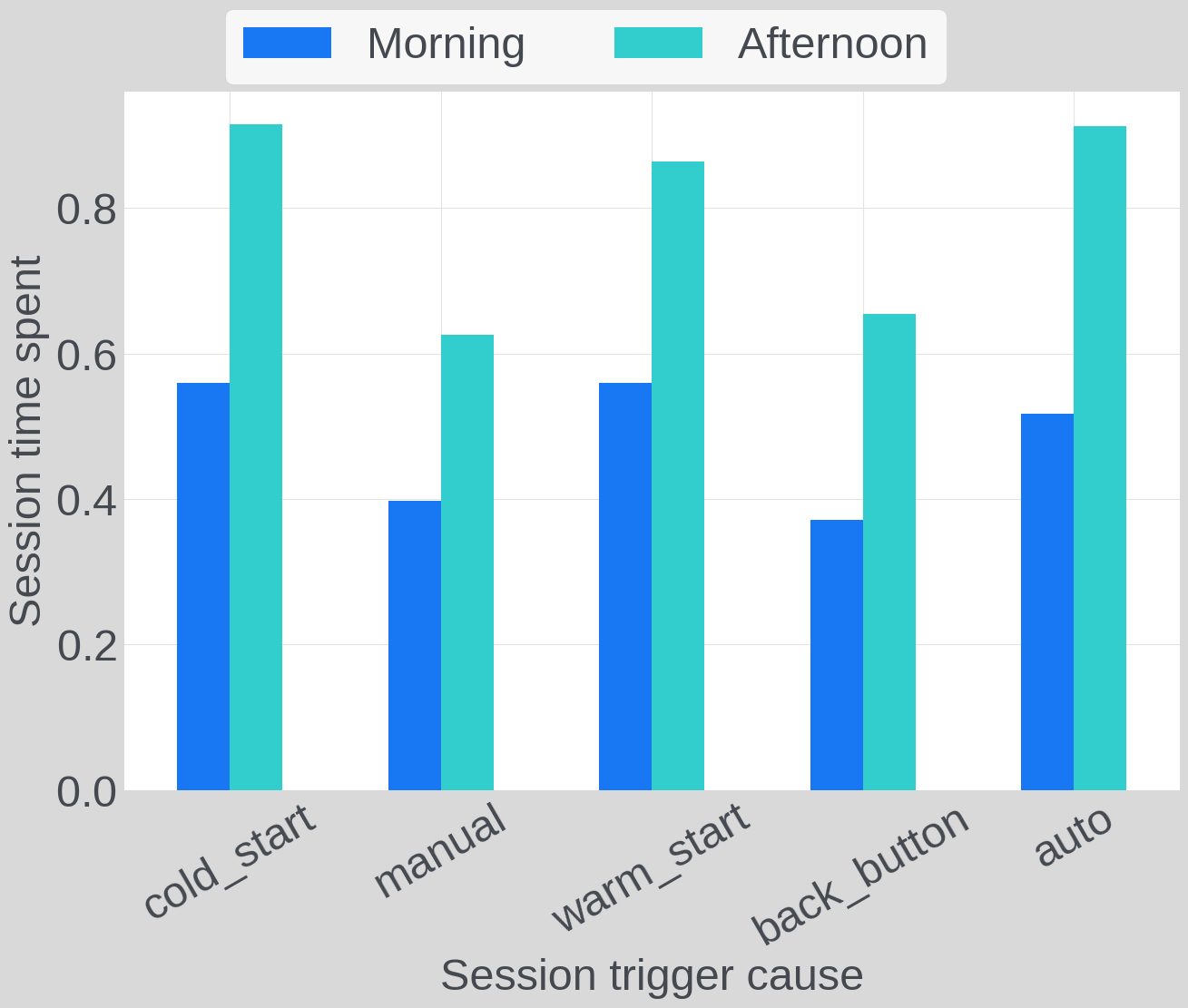}
    \caption{Mean value of engagement label (Y) versus time period by different trigger-cause (X) type.}
\end{subfigure}
\caption{Data analysis on confounding bias (X|T-shifts) and time-wise user behavior shift (Y|X-shifts)}
\label{fig:shift}
\end{figure}

For the session-level dataset, we measure algorithms with an AUCC metric to reflect ROI, where the x-axis is the normalized engagement loss and the y-axis is the normalized monetization gain. There are two common methods to rank individuals in the AUCC metric, i.e., sensitivity score and combined summation. The sensitivity score is defined as $-\Delta^{rev} / \Delta^{eng}$, while the combined summation is calculated as $\Delta^{rev} + \alpha \Delta^{eng}$, where $\alpha > 0$ and $\Delta$ measures the treatment effect on the monetization or engagement signal. Since DQN-based approaches are trained with a reward $r = r^{rev} + \alpha r^{eng}$, we will rank all algorithms based on combined summation in the decision module for fairness. 

In addition to production data, we also implement algorithms on the CartPole-v1 and LunarLander-v2 dataset \cite{brockman2016openai}, which are public data with similar properties as session-level production data (i.e., continuous state space and discrete action space). Due to the existence of simulators for the CartPole-v1 and LunarLander-v2 environments, we evaluate algorithms via cumulative rewards. Additionally, weight $\alpha = 1$ is chosen for production data, and radius $\delta=1\mathrm{e}{-4}$ is selected for all data.

\subsection{Production Data from the Same Distribution}
\label{subsec:prod-same}
The current baseline for session-level production data is meta-learner \cite{kunzel2019metalearners}, which suffers from confounding bias. In order to first verify the feasibility of the RL approach, we evaluate offline DQN and T-learner \footnote{The T-learner was selected as the baseline in the comparison with S-learner and X-learner due to its superior performance in our offline AUCC metrics and its suitability for our data characteristics. Our datasets exhibit strong treatment effect heterogeneity (i.e., variation in the effect of treatment across different individuals), which is not well-handled by S-learner. Additionally, the propensity score becomes unstable due to distribution shifts, which affects the effectiveness of the X-learner as it relies on propensity scores as a weighting function.} with XGBRegressor as a base learner on randomly split data across the whole production dataset. Specifically, 70\% data is used for training and the rest is used for testing. 

Since the average length of trajectory is around 5 in production data, we choose offline DQN with discount factor $\gamma=0.8$ as one of the candidates. The other candidate is offline DQN with $\gamma=0$, which is equivalent to neural contextual bandits and also shares a similar architecture as the treatment-agnostic representation network in causal learning literature \cite{shalit2017estimating}. 

As summarized in Table \ref{tab:same-dist-data}, offline DQN (1.1426 for $\gamma=0.8$) improves test AUCC by more than 80\% compared with the T-learner baseline (0.6049) on production data. Additionally, although AUCC is a short-term metric, the appropriate positive discount factor enjoys a better AUCC result, which may be due to avoiding some unusual edge cases that occur in a single session during training. Note that it is a possible phenomenon that the AUCC metric exceeds 1 since there are some counter-intuitive users who are more engaged after watching more advertisements.

\begin{table}
\caption{Summary of AUCC results for DQN and T-learner}
\centering
\begin{tabular}{lrr}
\toprule
Methods & Training AUCC & Test AUCC \\
\midrule
DQN, $\gamma=0.8$ & \textbf{1.2431} & \textbf{1.1426} \\
DQN, $\gamma=0.0$ & 1.0529 & 0.9570 \\
T-learner & 0.6214 & 0.6049 \\
\bottomrule
\end{tabular}
\label{tab:same-dist-data}
\end{table}

\begin{table}
\caption{Summary of AUCC results for different methods}
\centering
\begin{tabular}{lrr}
\toprule
Methods & Training AUCC & Test AUCC \\
\midrule
Robust dueling DQN & 0.7771 & \textbf{0.7290} \\
Dueling DQN & \textbf{0.7853} & 0.6962 \\
CQL & 0.6811 & 0.6601 \\
BCQ & 0.6718 & 0.6692 \\
T-learner & 0.6435 & 0.5750 \\
\bottomrule
\end{tabular}
\label{tab:diff-dist-data}
\end{table}

\subsection{Data with Distribution Shifts}
\label{subsec:data-diff}

\begin{figure*}[t]
\begin{subfigure}{0.33\textwidth}
    \centering
    \includegraphics[width=\textwidth]{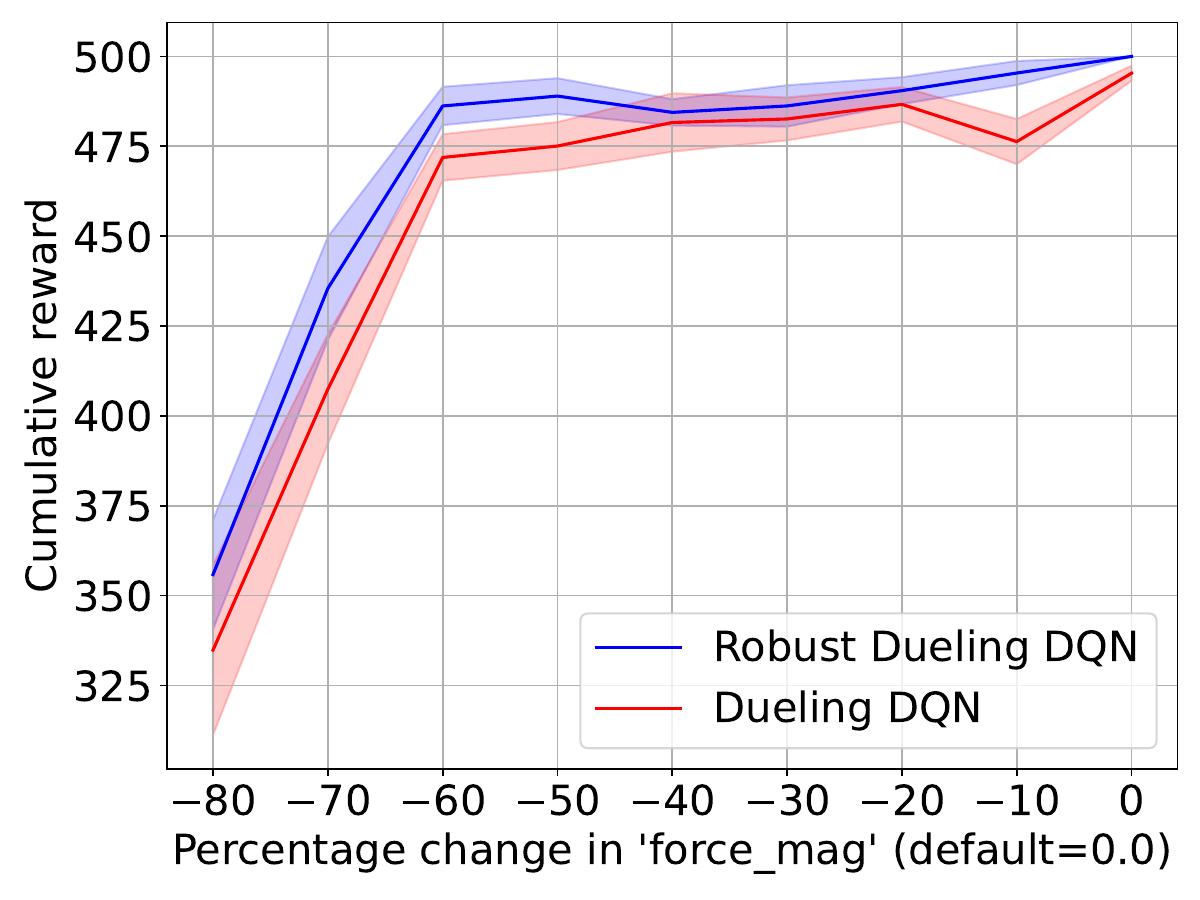}
    \caption{CartPole-v1 data}
\end{subfigure}
\begin{subfigure}{0.33\textwidth}
    \centering
    \includegraphics[width=\textwidth]{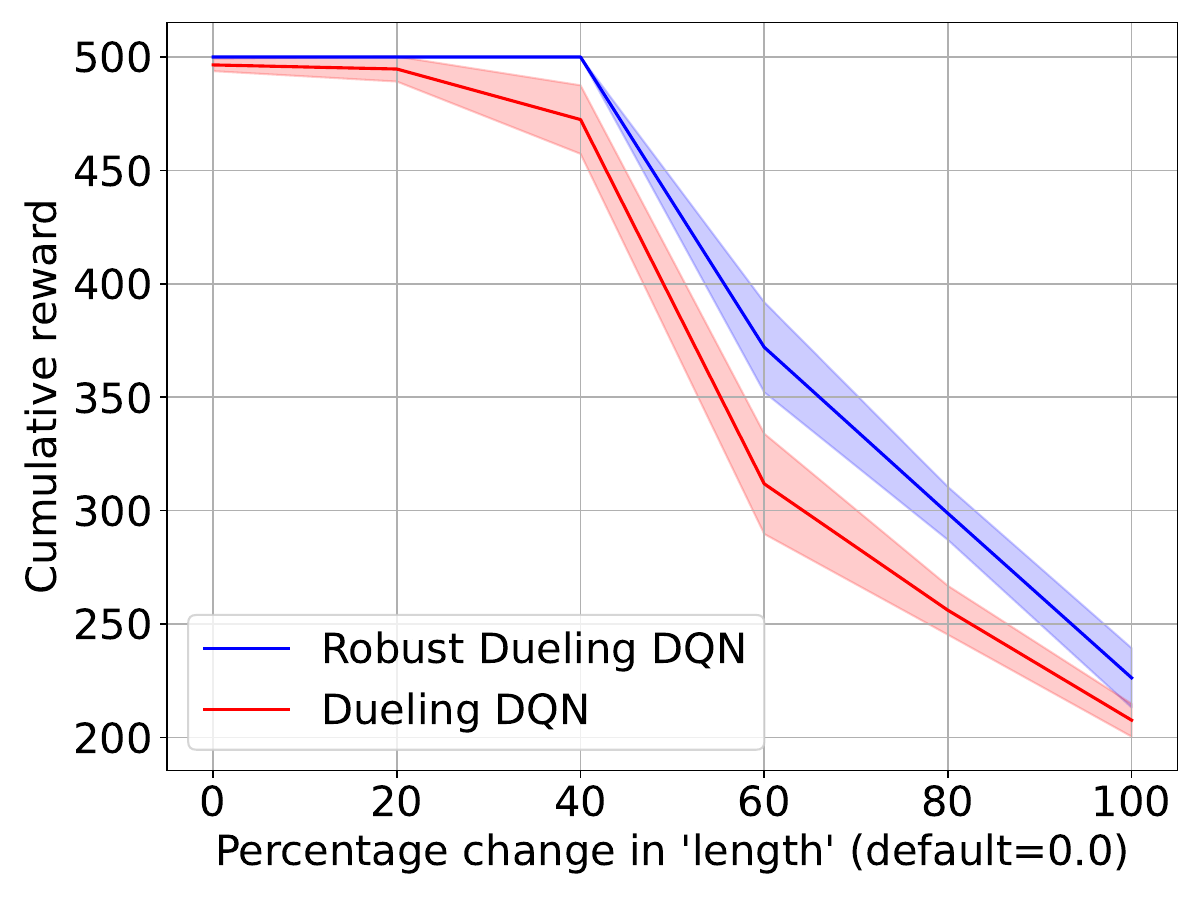}
    \caption{CartPole-v1 data}
\end{subfigure}
\begin{subfigure}{0.33\textwidth}
    \centering
    \includegraphics[width=\textwidth]{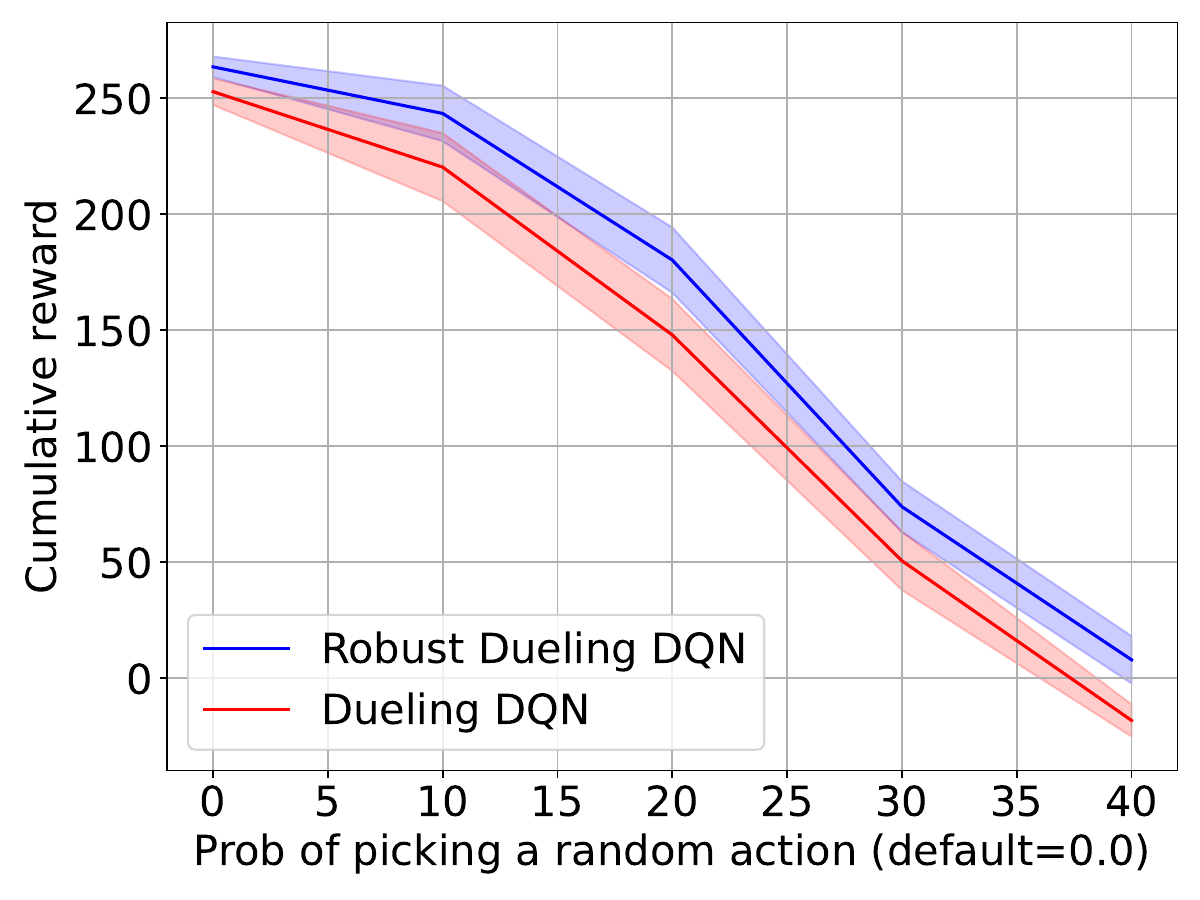}
    \caption{LunarLander-v2 data}
\end{subfigure}
\caption{Cumulative rewards of robust dueling DQN and dueling DQN under perturbation}
\label{fig:offline}
\end{figure*}

To verify the robust performance of the proposed robust dueling DQN, we evaluate several algorithms on both OpenAI Gym public data and session-level dynamic ad load production data. Note that these datasets are selected because they have key characteristics similar to our intended use case, i.e., session-level dynamic ad load optimization. For instance, they both feature continuous states and discrete action spaces, making them suitable and convenient for controlled testing of our algorithmic designs. We found that the findings on these datasets are highly transferable to our production scenarios, where distribution shifts can be simulated by different perturbations. 

\subsubsection{CartPole-v1 and LunarLander-v2 Public Data}
For training, we generate the offline dataset with $10^5$ samples using an $\epsilon$-greedy ($\epsilon=0.3$) version of proximal policy optimization (PPO) trained policy, which adds richness to the data by including non-expert behavior \cite{panaganti2022robust}. For evaluation, we record cumulative rewards on perturbed CartPole-v1 and LunarLander-v2 simulators by changing their physical parameter \textit{force\_mag} (to model external force disturbance), adding action perturbations (to model actuator noise), or altering its physical parameter \textit{length} (to model pole itself). 
{In our experiments, the robust dueling DQN and dueling DQN both are using neural network (NN) function approximation with two hidden layers of width 256. The decaying learning rate (LR) is configured as 0.1 * OLD\_LR + 0.9 * OLD\_LR * (1 - total\_steps / max\_train\_steps) with the initial learning rate 1e-4.}

We compare the performance of robust dueling DQN with that of the canonical dueling DQN algorithm in Figure \ref{fig:offline}, where the curves are averaged over 30 different seeded runs and the shaded region indicates the range of [mean - std, mean + std]. Compared to dueling DQN, robust dueling DQN enjoys robust behaviors with slow cumulative reward decay as perturbations increase on \textit{force\_mag}, \textit{length}, and action.

Moreover, we want to emphasize that the proposed robust algorithm takes almost the same training time as its non-robust counterpart, owing to its effective utilization of the IPM uncertainty set structure. In contrast, the SOTA methods such as \cite{panaganti2022robust}\footnote{For a fair comparison, in our experiments, both algorithms are using the same model architectures, i.e., dueling DQN instead of BCQ.} take about 50 times longer than that of our algorithm. \cite{panaganti2022robust} also fails to achieve a positive reward on LunarLander-v2, possibly due to its reliance on a restrictive Assumption 3, i.e., $\min_{s} V(s) = 0$, which doesn't hold in LunarLander-v2.

\subsubsection{Session-Level Dynamic Ad Load Production Data}
\label{subsubsec:-data}
The user behavior will change over time, which may bring severe Y|X-shifts. In other words, even if we adopt the same ad load strategy to similar user groups within similar sessions, the behavior of the next session may differ a lot. For prototyping, we use 1-day production data collected from the company's newsfeed platform and divide it by timestamp at 12:00 PM as the cut line for the training set (before 12:00 PM)  and the test set (after 12:00 PM).

Table \ref{tab:diff-dist-data} displays AUCC results of different approaches, where the discount factor $\gamma$ is chosen as 0.8 for all RL-based methods. Both dueling DQN and T-learner suffer significant performance drops, while robust dueling DQN enjoys better performance in test data with a different distribution. Specifically, the proposed robust dueling DQN improves test uplift AUC (0.7290) by ~5\% and ~25\% compared with non-robust dueling DQN (0.6962) and T-learner (0.5750), respectively. We also illustrate the results of two classic offline RL algorithms, batch-constrained Q-learning (BCQ) \cite{fujimoto2019off} and conservative Q-learning (CQL) \cite{kumar2020conservative}. As we mentioned in Section \ref{subsec:method}, BCQ and CQL don't perform well on production data since their pessimistic designs restrict the generalization power under the AUCC metric.

\subsection{Online Deployment}
\label{subsec:online}
\subsubsection{Model Setup} 
To date, we have successfully deployed 
our models to multiple product surfaces (e.g., News Feed) of a major social platform, to dynamically optimize session-level ad load to balance engagement and monetization topline goals. 
We adopt a teacher-student framework to develop the production models: 1) the proposed robust RL model served as the ``teacher model'' that provides the Q-value-based sensitivity score as the teacher label, and 2) a classification and regression tree (CART)-based regression model served as the ``student model'' to learn the mapping between input signals and the teacher labels, which is deployed for production serving. Compared to serving the DQN models directly, this design enables us to achieve comparable prediction quality while offering significantly better interpretability and serving efficiency. While the DQN-based RL models can achieve slightly better prediction accuracy, they often incur infrastructural complications and higher costs, hurting user experiences (e.g., significantly longer latency can cause apps to load much slower). Moreover, in our case, the complexity and cost are prohibited due to the scale of users we are serving. As real-time ad load decisions must be made on the fly after a user enters a session, the tree-based student model allows us to serve online requests at a massive scale with significantly better system performance (e.g., throughput, latency, and reliability). It also offers better interpretability, which is an important requirement to ensure operational maneuverability due to the nature of our ad delivery products. 

Taking one of the major product surfaces as an example (results are quite similar on different surfaces), we show in Table \ref{tab:teacher} that the student model achieved significantly better offline AUCC ($\sim$17\% gain) compared to a standalone tree model without learning from the teacher model. We choose to use the uplift tree here because regression models like CART could not directly learn the causal effect since there are no counterfactual labels without supervision from the teacher models. 
Moreover, we observe that the tree model, aided by the RL teacher model, performs comparably to the RL teacher model alone, but with a slight degradation of AUCC ($\sim$3\% lower).
This suggests that the teacher-student framework can effectively transfer knowledge from the teacher model to the students, ensuring the student models capture the causal relationship between user behavior and ad load decisions adequately.
\begin{table}
\caption{Ablation study for the teacher-student framework}
\centering
\begin{tabular}{lrr}
\toprule
Model type & Test AUCC\\
\midrule
Student model w/ teacher & 0.705 \\
Student model w/o teacher & 0.601  \\
Teacher model  & 0.729  \\
\bottomrule
\end{tabular}
\label{tab:teacher}
\end{table}

\subsubsection{Business Use Cases} There are two typical use cases for our models: 1) \textit{Ad score growth} -- to maximize ad score return (e.g., by increasing ad load) while keeping engagement loss under a set threshold; 2) \textit{Engagement recovery} -- to maximize engagement return (e.g., by reducing ad load) while keeping ad score target at a certain level. Depending on the stage that a product is at, the business use cases can often vary from product to product, and from time to time. 

Our models can be used to serve both purposes. For example, if engagement recovery is the goal, the same model can be used by changing the selection criteria towards the opposite direction as in the ad score growth scenario. 

Because ad score growth is relatively well-studied in the literature, we mainly discuss engagement recovery, although our deployed solutions include both. To measure the engagement impact of an ad load reduction policy accurately, we set up a control group by applying ad load reduction to randomly assigned sessions to keep ``ad impression'' at similar levels. Specifically, ad load is adjusted via various configurations in our ad delivery system. Besides ``ad impression'', we report online outcomes using ``ad score'' (i.e., the total financial value of ads combined with a quality value derived from various user experience indicators) as the ad score-related metric and ``time spent'' (i.e., the total duration user engages with the entire session) as the engagement-related metric. These metrics were chosen due to their strong long-term correlation with our core business goals (e.g., time spent is highly correlated to long-term engagement metrics such as daily active users and session counts). 

\subsubsection{Online Results}
Taking engagement recovery on the same product surface as an example, our production goal is to improve a business metric called \emph{engagement recovery efficiency}. This metric measures the ratio between engagement growth (e.g., increase in time spent) and ad score loss (e.g., decrease in ad score). As shown in Table \ref{tab:online}, the experimental group shows a +0.1\% gain in time spent with an almost neutral ad score cost (-0.002\%, non-statistically significant), resulting in an engagement recovery efficiency at 50. This is significantly higher than the control group (0.32), suggesting that our models are effective in optimizing engagement-ad score trade-offs when engagement recovery is the goal.

Although Table \ref{tab:online} only documents the post-launch backtest experiment conducted over three weeks, similarly positive trends are observed in subsequent three-month post-launch holdout experiments. Our framework has been deployed to multiple production systems to serve both ad score growth and engagement recovery scenarios. It has enabled us to achieve outsized topline business impacts. From the post-launch A/B test on live traffic, on average, we have observed a double-digit improvement in engagement-ad score trade-off efficiency. It has significantly improved our platform's capability to serve both consumers and advertisers effectively.

\begin{table}
\caption{Online readings of experimental groups (proposed approach) and control groups (random assignment) \tablefootnote{This was the initial deployment of ML models for session-level ad load optimization in our company, and thus, no existing control group option was available. Under this scenario, we selected the best candidate through offline experiments and conducted an A/B test against a random assignment.} by decreasing ad load at a similar level. Values are normalized.}
\centering
\begin{tabular}{lrr}
\toprule
Online metrics & Experimental & Control\\
\midrule
Ad impression & -0.52\% & -0.51\% \\
Ad score & -0.002\% & -0.25\% \\
Time spent & +0.1\% & +0.08\% \\
Engagement recovery efficiency & 50 & 0.32\\
\bottomrule
\end{tabular}
\label{tab:online}
\end{table}

\section{Conclusion and Discussion}
\label{sec:concl}
We prototype an offline DQN framework for session-level dynamic ad load optimization to mitigate the challenge of confounding bias and show its promising offline gains compared to current meta-learner approaches. Furthermore, to alleviate the distribution shift, we propose a new offline robust dueling DQN approach and demonstrate its robust performance on both OpenAI Gym public data and session-level dynamic ad load production data. The significant improvement of online A/B tests further illustrates the effectiveness of the proposed approach. 

There are several open directions we plan to explore in the future. For example, we are interested in off-policy evaluation as a long-term metric \cite{jiang2016doubly}. We also plan to build a high-fidelity simulator to better evaluate the performance of RL-based approaches. In addition, 
we are also exploring higher-order MDP-based formulations, which will enable us to model behavioral signals from the entire user's historical trajectory (instead of only the last state and action).
Moreover, instead of simple linear scalarization, we are also interested in more complex nonlinear scalarization between ad monetization and user engagement (e.g., proportional fairness, hard constraints, and max-min trade-off \cite{zhou2022anchor, liu2021policy, liu2021learning}) or determine the convex coverage set of the Pareto frontier \cite{roijers2013survey} based on different business needs. 

Although this paper focuses on session-level dynamic ad load optimization, the proposed approaches are applicable to general dynamic ad load optimization, including more fine-grained request-level optimization with a discrete action space \cite{xue2022resact}. The current DQN-based framework can handle discrete action spaces of any size since $\max_{a \in \mathcal{A}} Q(\cdot, a)$ is feasible in any discrete action space $\mathcal{A}$. Furthermore, to handle continuous action space, which typically requires parameterized policies that can predict continuous action values directly rather than selecting from a discrete set of options, we plan to explore alternative approaches such as deep deterministic policy gradient (DDPG) \cite{lillicrap2015continuous} or soft actor-critic method \cite{haarnoja2018soft} in our future work. For example, a robust DDPG can maintain the standard policy update of DDPG but integrate a new loss function for the robust Q-function update
\begin{align*}
    L_3(B, \theta) = \frac{1}{|B|} \sum_{i=1}^{|B|} (r_i + \gamma Q_{{\theta}^{target}}(s_i', \mu_{\phi^{target}}(s_i')) - \gamma \delta \|{\theta}^{target}_{2:d}\| \\
    - Q_{\theta}(s_i, a_i))^2,
\end{align*}
where $\phi^{target}$ and $\theta^{target}$ represent the target parameters for policy and Q-function, respectively.

\section*{Acknowledgement}
The authors would like to express their gratitude to colleagues Jizhe Zhang, Catherine Zhu, Weiyu Huang, Chen Fu, Yang Yang, Lin Gong, Zheqing Zhu, and Wei Lu for their insightful discussions and comments, which have greatly enhanced the content of this work. We are thankful for their time and effort in offering valuable feedback.

\bibliographystyle{ACM-Reference-Format}

%
\appendix
\newpage

\section{Area Under Cost Curve}
In this section, we demonstrate specific AUCC figures to supplement the advantages of the proposed approaches described in Sections \ref{subsec:prod-same} and \ref{subsec:data-diff}. Specifically, Figure \ref{fig:dqn-auuc} illustrates test AUCC results on session-level dynamic ad load production data from the same distribution, while Figure \ref{fig:robust-auuc} displays test AUCC on production data with distribution shifts.

\label{appsec:curve}
\begin{figure}[b]
    \centering
    \includegraphics[width=0.43\textwidth]{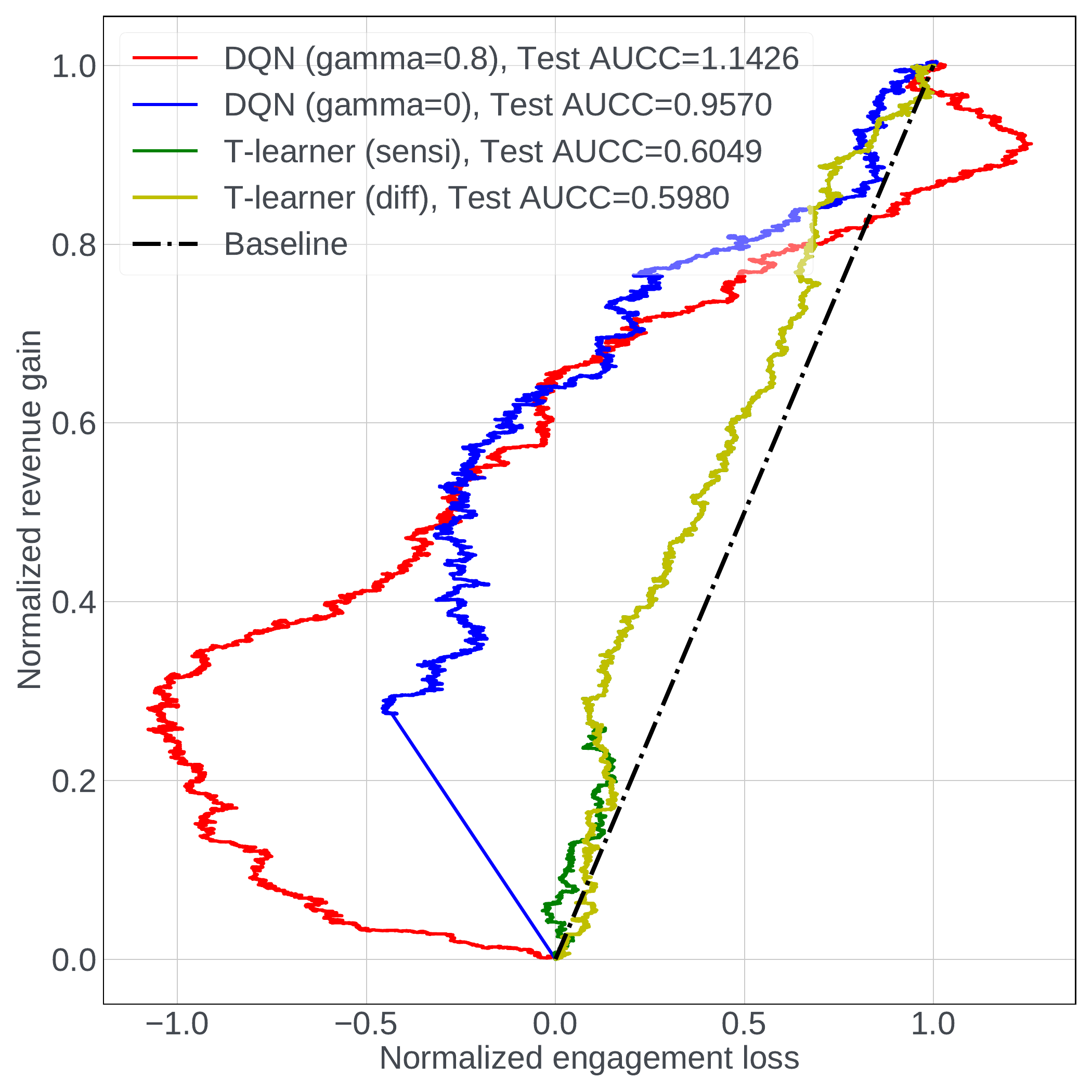}
    \caption{Test AUCC of offline DQN and T-learner on session-level production data.}
    \label{fig:dqn-auuc}
\end{figure}

\begin{figure}[b]
    \centering
    \includegraphics[width=0.43\textwidth]{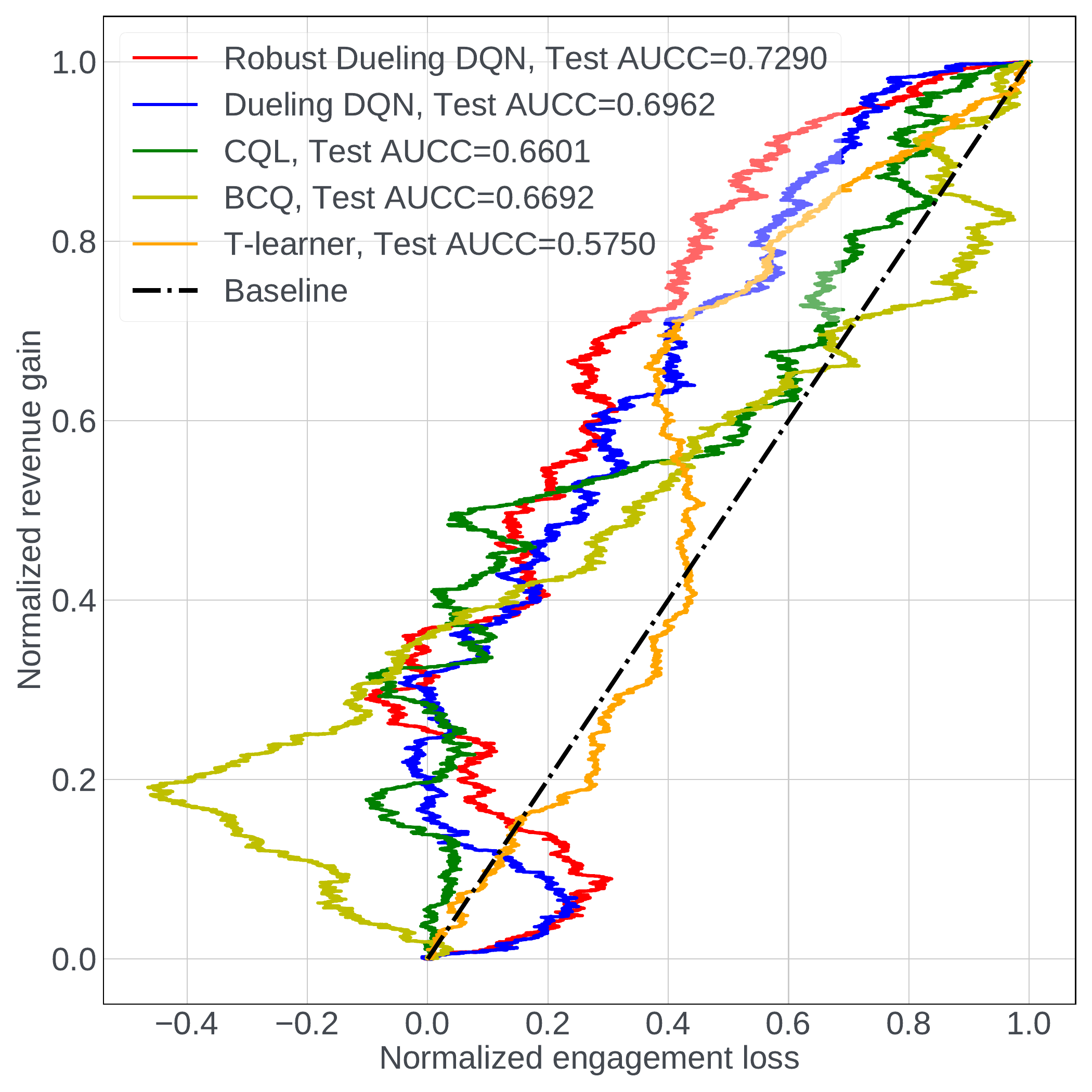}
    \caption{Test AUCC of robust dueling DQN, dueling DQN, CQL, BCQ, and T-learner on session-level production data.}
    \label{fig:robust-auuc}
\end{figure}

\section{Theoretical Convergence Guarantees}
Although no theoretical convergence guarantees can be derived for offline robust dueling DQN, its simplified version offline robust fitted Q-iteration (FQI) enjoys an approximate optimality under some mild assumptions (Theorem \ref{thm:rfqi-converge}). Specifically, the offline robust FQI with an IPM uncertainty set follows
\begin{align}
\label{eqn:rfqi-ipm}
    f_t \in \argmin_{f \in \Fc'} \sum_{i=1}^{|B|} \left(r_i + \gamma \max_{a_i' \in \Ac} f_{t-1}(s_i', a_i') - \gamma \delta \|w_{t-1, 2:d}\| - f(s_i, a_i) \right)^2,
\end{align}
where $\Fc' := \{(s, a) \mapsto f_w(s, a): w \in \Rb^d, \|w\| \leq 1, f_w(s, a) \in [0, 1/(1-\gamma)]\}$ and $\pi_t(s) = \argmax_{a \in \Ac} f_t(s, a), \forall s \in \Sc$. We then have the following theorem to demonstrate the convergence guarantee of offline robust FQI.

\begin{theorem}
\label{thm:rfqi-converge}
Assume $\frac{d^{\pi, P^{\pi}}(s, a)}{d^{\pi_{\beta}, P^0}(s, a)} \leq C, \forall \pi, s, a$, $\epsilon_{approx, d^{\pi_{\beta}, P^0}} := \max_{f \in \Fc} \min_{f' \in \Fc} \|f' - \Tc_{\Pc} f\|^2_{2, d^{\pi_{\beta}, P^0}}$, $\delta \leq 1/(1-\gamma)$, and $\Tc_{\Pc}$ is a $\beta$-contraction mapping w.r.t. the norm associated with any state-action distribution, then for any $T > 0$, robust FQI with an IPM uncertainty set (Equation (\ref{eqn:rfqi-ipm})) guarantees that with probability $1 - \delta$,
\begin{align}
    V_{\Pc}^* - V_{\Pc}^{\pi_T} \leq& \frac{1}{(1-\gamma)(1-\beta)} \left(\sqrt{\frac{88 C \ln (\frac{|\Fc'|^2 T}{\delta})}{N (1-\gamma)^2}} + \sqrt{20 C \epsilon_{apporx, d^{\pi_{\beta}, P^0}}}\right) \notag \\
    &+ \frac{\beta^T}{(1-\gamma)^2}.
\end{align}
\end{theorem}

\begin{proof}[Proof Sketch]
Theorem \ref{thm:rfqi-converge} can be proved following steps similar to those in the proof of Theorem 4.3 in \cite{agarwal2019reinforcement} with three different steps. The first is adopting robust performance difference lemma \cite{zhou2023natural} instead of performance difference lemma, i.e., for any state $s_0$ and policy $\pi, \pi'$, $V_{\Pc}^{\pi}(s_0) - V_{\Pc}^{\pi'}(s_0) \leq \frac{1}{1-\gamma} \Eb_{s \sim d_{s_0}^{\pi', \kappa'}} \Eb_{a \sim \pi'(\cdot|s)} [-A^{\pi}(s, a)]$, where $\kappa' := \arg\inf_{P \in \Pc} V_{P}^{\pi'}(s_0)$. The second is the fact that \\
$|\Tc_{\Pc} f(s, a)| \leq 1 + \frac{\gamma}{1-\gamma} + \delta \gamma \leq \frac{2}{1-\gamma}, \forall (s, a) \in \Sc \times \Ac$. The third is a contraction of the robust Bellman operator via the general function approximation version of Proposition \ref{prop:contract}.
\end{proof}

\textbf{Remark.} Similar to non-robust offline RL \cite{agarwal2019reinforcement, chen2019information, xie2021bellman}, the first two assumptions are necessary for the exploration power of data generating distribution $d^{\pi_{\beta}, P^0}$ and the representation power of general function class $\Fc'$. The third and fourth mild assumptions are for robust offline RL, which limits the radius of the uncertainty set to a reasonable range and extends Proposition \ref{prop:contract} to its general function approximation version. Compared with non-robust offline FQI convergence guarantees (cf. Theorem 4.3 in \cite{agarwal2019reinforcement}), Theorem \ref{thm:rfqi-converge} maintains a similar convergence rate. Although Theorem \ref{thm:rfqi-converge} is specific to offline robust FQI instead of offline robust dueling DQN, it provides some intuition that our design for robustness will hardly affect the convergence rate of non-robust algorithms under some mild assumptions, especially when the radius of the uncertainty set is small.

\end{document}